\documentclass[letterpaper, 10 pt, conference]{ieeeconf}  

\IEEEoverridecommandlockouts                              
\usepackage{cite}
\usepackage{amsmath,amssymb,amsfonts}
\usepackage{graphicx}
\usepackage{todonotes}
\usepackage{amsthm}
\usepackage[normalem]{ulem}
\usepackage{algorithm}
\usepackage{algpseudocode}

\newtheorem{theorem}{Theorem}[section]
\newtheorem{lemma}[theorem]{Lemma}
\newtheorem{assumption}[theorem]{Assumption}
\newtheorem{definition}{Definition}

\def\BibTeX{{\rm B\kern-.05em{\sc i\kern-.025em b}\kern-.08em
    T\kern-.1667em\lower.7ex\hbox{E}\kern-.125emX}}

\title{\LARGE \bf
Active Bayesian Inference for Robust Control under \\Sensor False Data Injection Attacks
}

\author{Axel Andersson and Gy\"{o}rgy D\'{a}n
\thanks{This work was partially supported by the Wallenberg AI, Autonomous Systems and
Software Program (WASP) funded by the Knut and Alice Wallenberg Foundation.}
\thanks{Authors are with the Department of Network and Systems Engineering, School of Electrical Engineering and Computer Science,
        KTH Royal Institute of Technology, Sweden
        {\tt\small \{axander|gyuri\}@kth.se}}%
}

\begin{document}

\maketitle
\thispagestyle{empty}
\pagestyle{empty}

\begin{abstract}

We present a framework for bridging the gap between sensor attack detection and recovery in cyber-physical systems. The proposed framework models modern-day, complex perception pipelines as bipartite graphs, which combined with anomaly detector alerts defines a Bayesian network for inferring compromised sensors. An active probing strategy exploits system nonlinearities to maximize distinguishability between attack hypotheses, while compromised sensors are selectively disabled to maintain reliable state estimation. We propose a threshold-based probing strategy and show its effectiveness via a simplified partially observable Markov decision process (POMDP) formulation. Experiments on an inverted pendulum under single and multi-sensor attacks show that our method significantly outperforms outlier-robust and prediction-based baselines, especially under prolonged attacks. 

\end{abstract}

\section{INTRODUCTION}
Cyber-Physical Systems (CPSs) integrate sensing, communication and physical processes, exposing attack surfaces in both the cyber and the physical domains. Sensor attacks, such as GPS spoofing \cite{gps-spoofing}, acoustic noise injection in gyroscopes \cite{gyro-attack} and electromagnetic interference on magnetometers \cite{magnetic-attack}, pose a serious threat to safety-critical CPSs such as unmanned aerial vehicles and autonomous cars. While progress has been made on detecting such attacks~\cite{choi-detecting-attacks, savior}, detection alone is insufficient. If the control system cannot appropriately respond to detections, the system remains vulnerable.

Existing approaches to attack recovery typically fall into two categories. The first is switching to a separate, pre-designed robust controller upon detection \cite{pid-piper, fast-attack-recovery}, which can be overly conservative and might not generalize well to new attack scenarios. The second relies on outlier-robust state estimation \cite{outlier-robust-ekfs} which scales the influence of suspicious measurements but does not explicitly identify compromised sensors. None of these approaches leverage the structural relationship between sensors and state estimates to isolate the attack. Furthermore, existing works overlook the complexity of modern state estimation pipelines that fuse heterogeneous sensors with different modalities and sampling rates, as exemplified by open-source flight controllers such as PX4~\cite{px4documentation}. 

In this work, we propose a framework that bridges the gap between attack detection and recovery by maintaining a belief over sensor attack states to selectively disable compromised sensors to preserve accurate state estimation. A key observation is that modern CPSs rely on multiple sensors of different modalities, whose measurements are fused through a perception pipeline. We model this pipeline as a bipartite graph that maps sensors to the components of the state estimate they influence. Combined with alerts from an existing anomaly detector, this graph structure defines a Bayesian network that enables inference over which sensors are compromised.

Beyond the passive inference framework, we introduce an \textit{active} probing strategy. By exploiting the non-linearity of the system dynamics, we design a probing control input  that  maximizes the distinguishability between state estimates obtained under different sensor configurations. We provide theoretical justification for the probing strategy by analyzing a simplified Partially Observable Markov Decision Process (POMDP) in which an agent must choose between a cheap and less accurate sensor and an expensive, more accurate one. We show that, under a Blackwell-dominance condition a probing policy based on a threshold structure in the belief state is a sensible choice. The system should probe when uncertainty about a sensor's integrity is high.

Our approach differs from \textit{dual control} in stochastic systems \cite{MesbahAli2018Smpc}, which seeks to simultaneously regulate the system while reducing parameter uncertainty. In contrast, we focus on identifying compromised sensors in adversarial settings. Additionally, the proposed framework incorporates a detailed model of the sensor processing pipeline. Our approach is related to active fault isolation  for leakage localization in water distribution networks~\cite{active-fault-detection}, but our approach differs in that we operate within a Bayesian belief framework over attack states, integrate the probing decision with a nominal tracking controller, and provide a theoretical analysis of when probing is optimal.

The contributions of this paper are:
\begin{itemize}
    \item \textbf{Perception graph model.} We model the perception pipeline as a bipartite graph relating sensors to state estimate components, enabling Bayesian inference over sensor attack states based anomaly detector alerts.

    \item \textbf{Detection-to-recovery framework.} We propose the LASE-AD algorithm, which maintains a belief over sensor attack states and selectively disables compromised sensors to preserve accurate state estimation, bridging the gap between attack detection and recovery.

    \item \textbf{Active probing for improved detection.} We design an information-seeking control input that exploits system non-linearity to increase the distinguishability between competing attack hypotheses. We further provide theoretical justification for a threshold-based probing policy on the belief state.
\end{itemize}

The rest of the paper is organized as follows. Section~\ref{sec::problem} introduces the system model, attack model and problem formulation. Section~\ref{sec::bayesian} presents the Bayesian inference framework for interpreting anomaly detector alerts. Section~\ref{sec::lase-ad} introduces the active learning strategy for attack detection and a simplified POMDP model motivating a threshold-based probing policy. Section~\ref{sec:numerical} presents an evaluation on an inverted pendulum and Section~\ref{sec::conclusion} concludes the paper.


\section{SYSTEM MODEL AND PROBLEM FORMULATION}
\label{sec::problem}
\subsection{System Model}
We consider a general, non-linear system of the form
\begin{equation}
    x_{k+1} = f(x_k) + g(x_k)u_k + w_k
    \label{eq:sys-dynamics}
\end{equation}
where $x_k \in \mathcal{X} \subseteq \mathbb{R}^{n_x}$ is the state, $u_k\in \mathcal{U} \subseteq \mathbb{R}^{n_u}$ is the control input, and $w_k \sim \mathcal{N}(\boldsymbol{0}, Q)$ is Gaussian process noise. A set of sensors is available, denoted by $\mathfrak{Z} = \{\mathfrak{z}_1, ... , \mathfrak{z}_m \}$. Each sensor is associated with a measurement function, and we denote the set of measurement functions by $\mathcal{H} = \{h_1, ..., h_m\}$, where each measurement function maps the state to a sensor-specific observation space, $h_i: \mathbb{R}^{n_x} \rightarrow \mathbb{R}^{p_i}$. For example, $h_i$ may represent the mapping that produces an RGB-image from a camera mounted on a UAV. Sensor measurements are noisy; the raw measurement generated by sensor $i$ at time step 
$k$ is
\begin{equation}
    y_{i,k}^r = h_i(x_k) + v_{i,k}, \quad v_{i,k} \sim \mathcal{N}(\boldsymbol{0}, R_i),
\end{equation}
where $v_{i,k}$ is the measurement noise and $ $ $R_i$ is the noise covariance matrix of sensor $\mathfrak{z}_i$. The collection of raw measurements from all sensors is denoted by $y_k^r = (y_{1,k}^r, ..., y_{m,k}^r) \in \mathcal{Y}_r$.

We assume the existence of a  perception pipeline  $\mathcal{P}$ that processes the raw sensor measurements $(y_{1,k}^r, ..., y_{m,k}^r)$ and produces a preprocessed measurement, referred to as a \textit{soft measurement}, $y_k = \mathcal{P}(y_{1,k}^r,...,y_{1,k}^r) \in \mathcal{Y} \subseteq \mathbb{R}^{n_x}$. The relationship between the soft measurement and the system state is modeled as
\begin{equation}
    y_k = Cx_k + \epsilon_k
    \label{eq:soft-measurement-state}
\end{equation}
where $C$ is assumed invertible and $\epsilon_k$ represents the modeling error of $\mathcal{P}$. 

The data flow in $\mathcal{P}$ is represented by a bipartite graph $\mathcal{G} = (\mathfrak{Z}\cup V_{y}, E)$, where $V_{y}$ is the set corresponding to the components of the soft measurement $y_k$. An edge $(\mathfrak{z}_i, v_j) \in E$ exists if sensor $\mathfrak{z}_i$ contributes to the estimation of component $j$ of $y_k$.

\subsection{Attack Model}
We consider a powerful attacker that can alter the raw sensor measurement of any sensor, resulting in measurement 
\begin{equation}
    y_{i,k}^a = y_{i,k}^r + b_{i,k}
    \label{eq:attack}
\end{equation}
where $b_{i,k}$ is an attack vector of appropriate size.
The attacker can alter the measurements of multiple sensors simultaneously, the collection of attack vectors will be denoted by $b_k = \{b_{i,k}\}_{i=1}^{m}$. We assume that the attacker cannot manipulate the control input 
$u_k$ or modify the system's source code to disable the anomaly detection mechanism or the proposed framework.

\subsection{State Estimation and Anomaly Detector}
Different soft measurements can be constructed by selectively disabling sensors. Let $\mathcal{P}_S$ denote a perception pipeline in which only the sensors in $S\subseteq\mathfrak{Z}$ are enabled. For each $S \in 2^{\mathfrak{Z}}$, the corresponding pipeline produces soft measurements $y_k^S$ satisfying an equation on the form of (\ref{eq:soft-measurement-state}). The pipeline that uses all sensors is denoted simply by $\mathcal{P}$. State estimation is done with an extended Kalman filter (EKF) that is updated updated with the soft measurements, $y_k$.

An EKF-based anomaly detection algorithm $\mathcal{D}$, such as that proposed in \cite{savior}, is assumed to detect deviations in $y_k$ by comparing it to the filter output. The detector is modeled as a mapping $\mathcal{D}: \mathcal{Y} \times \mathcal{U} \rightarrow \mathcal{A} =\{0,1\}^{n_x}$, which produces element-wise alerts for the soft measurement, 
$a_k = (a_{1,k},\ldots,a_{n_x,k})$, where $a_{i,k}$ indicates an alert for component $i$ of $y_k$. 
let $z_{i,k} \in \{0,1\}$ denote whether sensor $\mathfrak{z}_i$  is under attack at time $k$ (i.e., $b_{i,k} \neq 0$ in eq. (\ref{eq:attack})). 
Similarly, let $s_{j,k}\in\{0,1\}$ denote whether component $j$ of $y_k$ is compromised at time $k$. We assume a conservative causal attack propagation model, 
\begin{equation}
    s_{j,k} = \max_{i \in N(j)}z_{i,k}
    \label{eq:s-def}
\end{equation}
where $N(j) = \{i: (i,j)\in E \}$ denotes the set of neighbors to vertex $j$. That is, if any of the sensors that contribute to the estimation of component $j$ is attacked, then the corresponding component is considered compromised. 

We model the alert generation process for each component as a two-state Markov process conditioned on the attack variable $s_{i,k}$.
The anomaly detector $\mathcal{D}$ is assumed to satisfy the following transition probabilities: 
\begin{subequations}
    \begin{align}
        \mathbb{P}(a_{i,k} = 1 | a_{i,k-1}=0,s_{i,k} = 0) &= \eta_{i,0} \label{eq:eta0} \\
        \mathbb{P}(a_{i,k} = 1 | a_{i,k-1}=1,s_{i,k} = 0) &= \eta_{i,1} \label{eq:eta1} \\
        \mathbb{P}(a_{i,k} = 0 | a_{i,k-1}=0,s_{i,k} = 1) &= \xi_{i,0} \label{eq:xi0} \\
        \mathbb{P}(a_{i,k} = 0 | a_{i,k-1}=1,s_{i,k} = 1) &= \xi_{i,1}. \label{eq:xi1}
    \end{align}
\end{subequations}
Equations (\ref{eq:eta0}, \ref{eq:eta1}) characterize false alarm probabilities and (\ref{eq:xi0}, \ref{eq:xi1}) the missed detection probabilities. The dependence on the previous alert $a_{i,k-1}$ captures temporal correlations in detector outputs, which arise naturally in sequential detection methods such as CUSUM~\cite{signal-detection-and-parameter-estimation}.  The false alarm probabilities $\eta_{i,0}$ and $\eta_{i,1}$ are assumed to be known, as they can be estimated by evaluating $\mathcal{D}$ on benign data. The missed detection probabilities are harder to estimate in lack of representative attack data. We therefore 
 treat $\xi_{i,0}$ and $\xi_{i,1}$ as unknown, but we assume a Beta distributed prior,
$\xi_{i,\cdot} \sim \text{Beta}(\beta_{i,\cdot}^1, \beta_{i,\cdot}^2)$. The resulting model can be interpreted as a graphical model where the latent attack variable $s_{i,k}$ influences the alert variable $a_{i,k}$, while temporal correlation is captured through the Markovian dependence on $a_{i,k-1}$.

\subsection{Problem Formulation}
We consider a control system described by (\ref{eq:sys-dynamics}), equipped with a set of sensors $\mathfrak{Z}$, a family of perception pipelines $\{\mathcal{P}_S : S \in 2^{\mathfrak{Z}} \}$ and an anomaly detection algorithm $\mathcal{D}$. Our objective is to enable recovery from sensor attacks. 
We formulate this as a reference tracking problem in which a policy $\kappa: \mathcal{A} \times \mathcal{Y}_r \rightarrow 2^\mathfrak{Z}$ maps alerts and raw measurements to a subset of trusted sensors. At each time step $k$, the policy $\kappa$  selects a set of trustworthy sensors $S_k$ and the corresponding perception pipeline $\mathcal{P}_{S_k}$ is used to construct the soft measurement. We want to solve the following problem:
\begin{subequations}
    \begin{align}
        \min_{\kappa} &\mathbb{E}\Big[\sum_{k=0}^\infty \gamma^k(x_k^{\top}M_xx_k + u_k^\top M_uu_k) \Big] 
        \label{eq:ref-opt-problem}
        \\
        \text{s.t} \quad &x_{k+1} = f(x_k)+g(x_k)u_k+w_k \\
        \quad &\mathbb{E}(x_k) \in \mathcal{X}_{\text{safe}}
        \label{eq:safety-ref-opt-problem}
        \\
        &u_k \in \mathcal{U} \quad \quad \forall k \in \mathbb{N},
        \label{eq:saturation-ref-opt-problem}
    \end{align}
    \label{eq:main-opt-problem}
\end{subequations}
\noindent where $M_x\in \mathbb{R}^{n_x \times n_x}$, $M_u\in \mathbb{R}^{n_u \times n_u}$ and $\gamma \in (0,1)$ is the discount factor. Constraint (\ref{eq:safety-ref-opt-problem}) ensures system safety in expectation and (\ref{eq:saturation-ref-opt-problem}) models hardware limitations by forcing the control input to belong to a bounded set, $\mathcal{U}$. The reference signal should be tracked accurately, even in the event of a sensor attack of the form (\ref{eq:attack}). 

We propose to address the problem in a Bayesian dual control framework: we maintain a belief over sensor attack states, and use an active learning strategy that excites the system to generate informative observations for belief refinement whenever beneficial. While (\ref{eq:main-opt-problem}) can be formulated as a partially observable Markov decision process (POMDP), solving it exactly is computationally intractable for real-time implementation. Our approach instead uses tractable approximations to retain the key benefits of dual control without the complexity of solving the full POMDP.

\section{Bayesian Inference for Sensor Attack Identification}
\label{sec::bayesian}

To address (\ref{eq:main-opt-problem}), we first  establish a relationship between alerts generated by the anomaly detection system and the attack status of each sensor. Let $\pi_{i,k} := \mathbb{P}(z_{i,k} = 1)$ denote the prior belief that sensor $\mathfrak{z}_i$ is under attack at time $k$. For tractability, we assume that the  attack status variables at a given time are independent, i.e., the joint probability of an attack $z_k=(z_{1,k},...,z_{m,k})$ factorizes as
\begin{equation}
    \mathbb{P}(z_k = z') = \prod_{i=1}^m\pi_{i,k}^{z'_i}(1-\pi_{i,k})^{(1-z'_i)}.
    \label{eq:Pz_k}
\end{equation}
The independence assumption is reasonable as a powerful attacker could target any sensor. 
Upon receiving a set of alerts from the anomaly detector $\mathcal{D}$, $a_k = \{a_{i,k} \}_{i=1}^{n_x}$, the posterior belief $\pi_{i,k}(a_k) := \mathbb{P}(z_{i,k}=1|a_k, a_{k-1})$ can be computed by combining the causal relationship between sensor states and soft measurements in (\ref{eq:s-def}), with the detector performance characterized in (\ref{eq:eta0}-\ref{eq:xi1}). Note that (\ref{eq:s-def}) and (\ref{eq:eta0}-\ref{eq:xi1}) together with the internal structure of $\mathcal{P}$ captured by the graph $\mathcal{G}$, define a Bayesian network (BN). An example of such a BN is shown in Fig.~(\ref{fig:bayesiannet}). 
\begin{figure}
    \centering
    \includegraphics[width=0.9\linewidth]{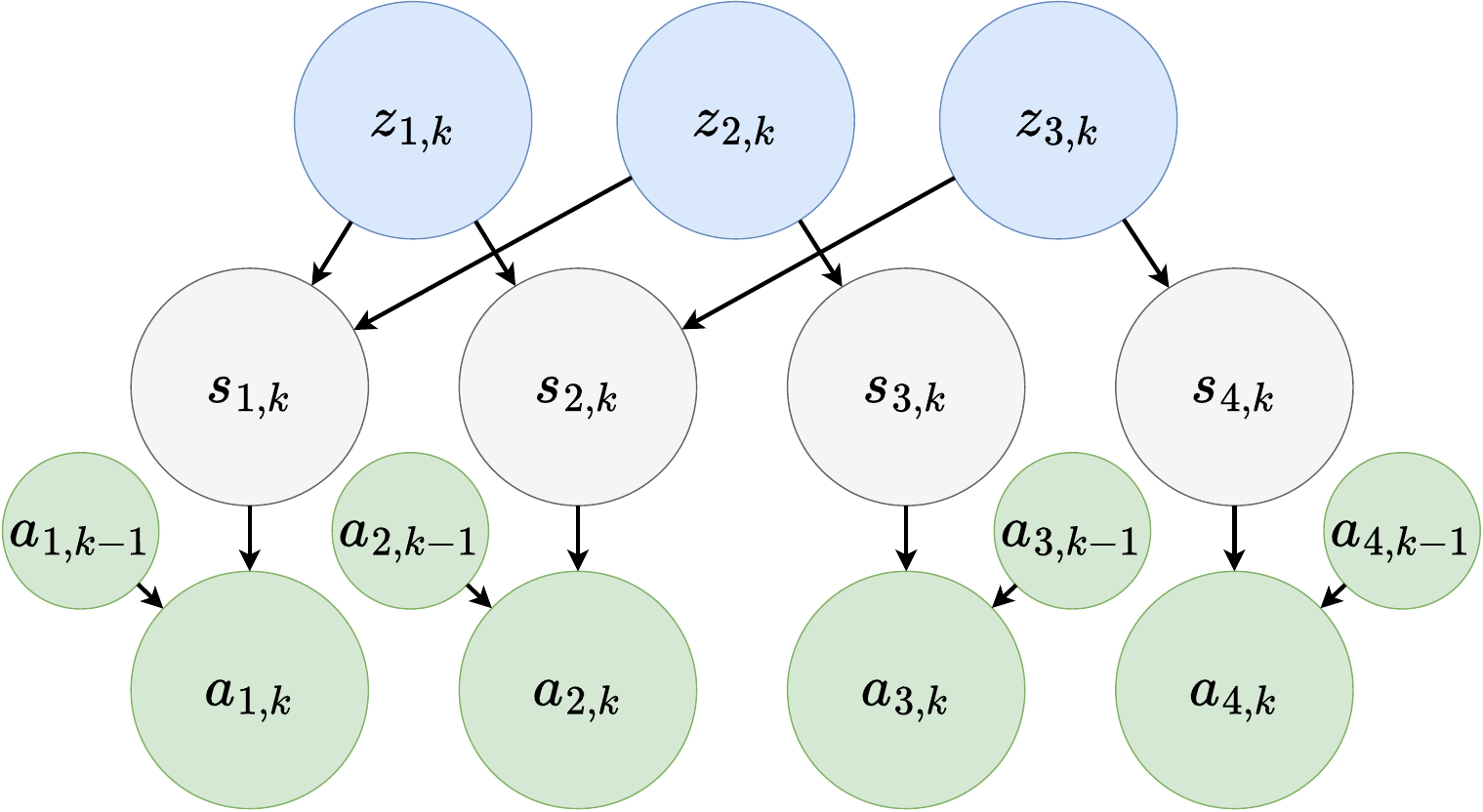}
    \caption{Bayesian network of sensor attack status variables ($z_k$), state attack status variables ($s_k$), and alerts ($a_k, a_{k-1}$) for the cart-pole problem with a wheel encoder, a camera, and an IMU as sensors (c.f. Section~\ref{sec:numerical}).}
    \label{fig:bayesiannet}
\end{figure}
The posterior $\pi_{i,k}(a_k)$ can be obtained from probabilistic inference on the Bayesian network,
\begin{align}
    &\pi_{i,k}(a_k) := \mathbb{P}(z_{i,k}|a_k, a_{k-1}) = \\
    &\frac{\pi_{i,k}\mathbb{P}(a_k|z_{i,k}=1, a_{k-1})}{\pi_{i,k}\mathbb{P}(a_k|z_{i,k}=1, a_{k-1}) + (1-\pi_{i,k})\mathbb{P}(a_k|z_{i,k}=0, a_{k-1})},
    \label{eq:alert-belief-update}
\end{align}
where
\begin{align*}
    &\mathbb{P}(a_k|z_{i,k}=z', a_{k-1}) = \\
    &\quad\quad\sum_{z_{-i,k} \in \{0,1\}^m } \mathbb{P}(z_{-i,k})\mathbb{P}(a_k | z_{i,k}=z', z_{-i,k}, a_{k-1}),
\end{align*}
and $\mathbb{P}(z_{-i,k})$ is given by an expression similar to (\ref{eq:Pz_k}) but without the $i$:th component. Furthermore, from the Bayesian network we have
\begin{align}
    \nonumber & \mathbb{P}(a_k|z_k,a_{k-1}) = \\
    \nonumber & \mathbb{P}(a_k=a|s_k=s,a_{k-1}=a')\mathbb{P}(s_k=s|z_k=z) = \\
    \;\;\nonumber & \sum_s \boldsymbol{1}_{[s_i = \max_{j\in N(i)}z_j \quad \forall i \in \{1,...,n_x\}]} \cdot \\
    & \;\;\prod_{i=1}^{n_x} 
    \begin{cases}
        \eta^{a_i}_{i,a'_i}(1-\eta_{i,a'_i})^{1-a_i}, & s_{i,k}=1 \\
        \int_0^1 \xi^{1-a_i}_{i,a'_i}(1-\xi_{i,a'_i})^{a_i}\text{Beta}(\beta^1_{i,a'}, \beta^2_{i,a'})d\xi, & s_{i,k}=0
    \end{cases}
    \label{eq:bayesnet-fp-fn}
\end{align}
At every time step $k$ the belief $\pi_k$ is computed based on alerts $a_k$ from the anomaly detector, using the posterior $\pi_{i,k-1}$ as the prior to $\pi_{i,k}$. 

\section{LEARNING-AUGMENTED STATE ESTIMATION FOR SENSOR ATTACK DETECTION}
\label{sec::lase-ad}
Next, we design an active learning scheme that excites the system to accelerate the identification of attacked sensors.

\subsection{Information-Optimal Control}

\label{sec:active-learning}
Eqns.~\eqref{eq:Pz_k}-\eqref{eq:bayesnet-fp-fn} allow us to estimate sensor attack probabilities $\pi_{i,k}$ based on alerts $a_k$ from the anomaly detector $\mathcal{D}$. 
A complementary approach is to \textit{actively} probe for attacks through the control input $u_k$. In particular, consider two hypotheses, $h_0$ and $h_1$, corresponding to different sensor attack assumptions, resulting in state estimates $\hat{x}_k^{h_0}$ and $\hat{x}_k^{h_1}$ obtained using two perception pipelines.
The predicted system output and consequently the innovation of the EKF follows a different distribution under each hypothesis. The innovation at time $k+1$ is given by
\begin{equation}
    r_{k+1} = y_{k+1}-C\hat{x}_{k+1:k},
\end{equation}
where $\hat{x}_{k+1:k}$ is the EKF state prediction at time $k$ about the state at time $k+1$. 
We design the probing control input to maximize the Kullback–Leibler (KL) divergence between the innovation distributions under the two hypotheses. Specifically, we consider $KL(p(r_{k+1}|h_0) ||p(r_{k+1}|h_1))$, which measures the expected log-likelihood ratio under the null-hypothesis $h_0$. 
Maximizing this quantity makes the innovations as informative as possible for discriminating between the hypotheses. This connects directly to sequential hypothesis testing: a larger KL divergence implies faster accumulation of evidence, and thus quicker detection \cite{signal-detection-and-parameter-estimation}.
Since $r_{k+1}$ is Gaussian under both $h_0$ and $h_1$, the KL divergence admits a closed-form expression
\begin{align}
   \nonumber &KL(p(r_{k+1}|h_0) || p(r_{k+1}|h_1)) = \\
    &\frac{1}{2}\Big[ ||r_{k+1}^{h_1}-r_{k+1}^{h_0}||_{(P_{k+1:k}^{h_1})^{-1}}^2 + c(P^{h_1}_{k+1:k}, P^{h_0}_{k+1:k})\Big] \text{ and} \\
    &r_{k+1}^{h_1}-r_{k+1}^{h_0} = F_k^{h_0,h_1} + G_k^{h_0, h_1}u_k,
\end{align}
The matrices are given by $F_k^{h_0,h_1} = C(f(\hat{x}_{k}^{h_0}) - f(\hat{x}_{k}^{h_1}))$ and $G_k^{h_0, h_1} =C(g(\hat{x}_{k}^{h_0})-g(\hat{x}_{k}^{h_1}))$. Note that if $G_k^{h_0, h_1}$ were not a function of $x_k$, the control input $u_k$ would not have any influence on the innovation difference as $G_k^{h_0, h_1}=0$. $||\cdot||_P$ is the Mahalanobis norm with matrix $P$ and $c(P^{h_1}_{k+1:k}, P^{h_0}_{k+1:k}) = \text{tr}((P^{h_1}_{k+1:k})^{-1}P^{h_0}_{k+1:k})-\ln|P^{h_0}_{k+1:k}|/|P^{h_1}_{k+1:k}|-n_x \approx 0$ if the covariance matrices are similar.

We thus solve the following optimization problem to obtain the probing control signal
\begin{subequations} \label{eq:opt-problem}
    \begin{align}
        \max_{u_k} & KL(p(r_{k+1}|h_0) || p(r_{k+1}|h_1)) \\
        \text{s.t} \quad & f(\hat{x}_{k}^{h_0} ) + g(\hat{x}_{k}^{h_0})u_k \in \mathcal{X}_{\text{safe}} \label{eq:opt-constraint1} \\
        \quad & f(\hat{x}_{k}^{h_1} ) + g(\hat{x}_{k}^{h_1})u_k \in \mathcal{X}_{\text{safe}} \label{eq:opt-constraint2}\\
        \quad &  u_k \in \mathcal{U}.
    \end{align}
    \label{eq:probing-ctr-signal}
\end{subequations}
Problem (\ref{eq:probing-ctr-signal}) finds the control input that maximizes the KL divergence between the innovation distributions under the two hypotheses, while enforcing that the system remains within the safe set under both hypotheses when the input is applied. In practice, we assume the innovation covariance matrices are similar, in which case the KL divergence reduces to the Mahalanobis distance between the predicted innovation means.

The null-hypothesis uses the nominal pipeline $\mathcal{P}$, while probing sensor $\mathfrak{z}_i$ corresponds to the pipeline $\mathcal{P}_{\mathfrak{Z}\setminus \{\mathfrak{z}_i \}}$. A FIFO-buffer of length $T_{\text{r}}$ stores past estimates and measurements, allowing $\hat{x}_k^{\mathfrak{Z}\setminus \{\mathfrak{z}_i \}}$ to be reconstructed by re-running an EKF from $\hat{x}_{k-T_{\text{r}}}$ under $\mathcal{P}_{\mathfrak{Z}\setminus \{\mathfrak{z}_i \}}$. We denote the state estimate produced this way using perception pipeline $\mathcal{P}_S$ by $\hat{x}_k^S$.

The effect of applying the probing control input obtained from solving (\ref{eq:probing-ctr-signal}) can be exploited at the next time step to update the attack belief for sensor $\mathfrak{z}_i$, using the likelihood ratio
\begin{align}
     \pi_{i,k} \leftarrow \frac{1}{1 + \frac{1-\pi_{i,k}}{\pi_{i,k}} \frac{\mathbb{P}(y_k|z_{i,k}=0)}{\mathbb{P}(y_k|z_{i,k}=1)} } \approx 
    \frac{1}{1 + \frac{1-\pi_{i,k}}{\pi_{i,k}} \frac{\mathcal{N}(y_k; \mu_{h_1}, \Sigma_{h_1})}{\mathcal{N}(y_k;\mu_{h_0}, \Sigma_{h_0})} },
    \label{eq:probing-update}
\end{align}
where $(\mu_{h_0}, \Sigma_{h_0}) = (C\hat{x}_{k:k-1}, C^\top P_{k:k-1}C+R)$ and $(\mu_{h_1}, \Sigma_{h_1}) = (C\hat{x}_{k:k-1}^{\mathfrak{Z}\setminus \{ \mathfrak{z}_i \}}, C^\top P_{k:k-1}^{\mathfrak{Z}\setminus \{ \mathfrak{z}_i\}}C+R_{\mathfrak{Z}\setminus \{ \mathfrak{z}_i\}} )$. Here, the $P$-matrices denote  EKF covariance matrices and $R$ are the measurement noise covariance matrices. 

If sensor $\mathfrak{z}_i$ is not attacked, the likelihood-ratio in (\ref{eq:probing-update}) should be close to one. However, if $\mathfrak{z}_i$ has been attacked for some time, the nominal state estimate $\hat{x}$ may become contaminated, while $\hat{x}^{\mathfrak{Z} \setminus \{ \mathfrak{z}_i \}}$ remains unaffected. In this case, the likelihood ratio is small, leading to an increase of  the attack belief.

\subsection{Inference-aware Control Scheduling}
Using the probing control input introduces a trade-off: while the objective in (\ref{eq:main-opt-problem}) is to minimize the control error,  the nominal controller is designed to achieve this objective in the absence of attacks. Applying a probing input may temporarily deviate the system from its reference trajectory, but it can accelerate the identification of compromised sensors and improve state estimation, potentially outweighing the short-term performance degradation.

Instead of reformulating problem (\ref{eq:main-opt-problem}) as a POMDP,  we propose a threshold policy: probe sensor $\mathfrak{z}_i$ if $\pi_{i,k}\in (\underline{\pi}, \overline{\pi})$ and use the nominal controller otherwise. This policy captures the intuition that probing is most informative when the uncertainty about a sensor's integrity is highest. The pseudo-code of the proposed \textbf{L}earning-\textbf{A}ugmented \textbf{S}tate \textbf{E}stimation for \textbf{A}ttack \textbf{D}etection (LASE-AD) algorithm is shown in Algorithm \ref{alg:lase-ad}.

In what follows, we justify the threshold policy  via a simplified POMDP, 
where a binary Markov chain $z_k$ (analogous to sensor attack state $z_{i,k}$) is observed through either a cheap, less accurate sensor (nominal control input) or an expensive, more accurate one (active probing). The objective is to minimize misclassification cost plus sensor usage cost, serving as a surrogate to (\ref{eq:main-opt-problem}). We show that the optimal policy probes whenever a threshold policy does, establishing it as a conservative approximation. 

\subsubsection{Simplified POMDP Model of Active Sensing}
\label{sec:simple-pomdp}
Let $\{z_k\}_{k=0,1,2...}$ be a time-homogeneous Markov chain on the binary state space $\mathcal{Z} = \{0,1\}$ with transition matrix
\begin{equation}
    A = \begin{bmatrix}
    a_{00} & a_{01} \\
    a_{10} & a_{11}
    \end{bmatrix}, \quad a_{ij} = \mathbb{P}(z_{k+1}=j|z_k=i),
\end{equation}
such that $a_{i1} = 1-a_{i0}$ for $i = 0,1$. Furthermore, we have that $a_{00}, a_{11} \in (0,1)$ so that the Markov chain is ergodic.

The state variables $z_k$ are not directly observable. At each time step $k$, an agent can select either a Cheap or an Expensive sensor $s_k\in\{C,E\}$ and receives a binary observation $o_k\in\{0,1\}$. The observation likelihoods are parametrized by their false positive rate $\alpha_s$ and true positive (detection) rate $\tau_s = 1-\beta_s$,
\begin{subequations}
\begin{align}
    &\mathbb{P}(o_k = 1 | z_k = 0, s_k=s) = \alpha_s \\
    &\mathbb{P}(o_k = 1 | z_k = 1, s_k=s) = \tau_s
\end{align}
\end{subequations}
We make the following assumption on the sensors.
\begin{assumption}
    (Sensor ordering) Sensor E is strictly more informative than Sensor C in the sense that
    \begin{equation*}
    \alpha_E < \alpha_C \quad and \quad \tau_E > \tau_C
    \end{equation*}
    and both sensors are non-trivial: $\alpha_s < \tau_s$ for $s \in \{C,E\}$.
    \label{A:sensor-ordering}
\end{assumption}

\begin{lemma}
    Assumption \ref{A:sensor-ordering} is equivalent to Sensor C being a garbling of Sensor E. That is, there exists a stochastic matrix $G$ such that $O_C=O_EG$, where $O_s$ is the observation matrix of sensor $s$ with entries $(O_s)_{ij} = \mathbb{P}(o_k=i|z_k=j, s_k=s)$. 
    \label{L:Blackwell-dominance}
\end{lemma}
\begin{proof}
See Theorem 12.4.2 in \cite{blackwell-book}.
\end{proof}

Let $\pi_k := \mathbb{P}(z_k=1|o_k, s_k,...,o_0,s_0)$ denote the posterior probability that the hidden state $z_k=1$, given all observations and sensor choices up until time $k$. The belief update can be done in two steps, using a Prediction step followed by an Update step (see e.g., \cite{russel2010}).

\noindent \textbf{Prediction Step.} Given the posterior $\pi_{k-1}$ from the previous step, the prediction belief before the observation is obtained at time $k$ is 
\begin{equation}
    \pi_k^- = T(\pi_{k-1}) := a_{01}(1-\pi_{k-1})+a_{11}\pi_{k-1}.
\end{equation}
Note that $T: [0,1] \rightarrow [0,1]$ is an affine map, $T(\pi)=a_{01} + (a_{11}-a_{01})\pi$.

\noindent \textbf{Update Step.} After choosing sensor $s_k = s$ and observing $o_k = o$, the posterior is
\begin{equation}
    \pi_k = U_s(\pi_k^-, o) := \begin{cases}
        \frac{\tau_s\pi_k^-}{\tau_s\pi_k^- + \alpha_s(1-\pi_k^-)} \quad &\text{if }o=1 \\
        \frac{(1-\tau_s)\pi_k^-}{(1-\tau_s)\pi_k^- + (1-\alpha_s)(1-\pi_k^-)} \quad &\text{if }o=0. 
    \end{cases}
    \label{eq:general-update}
\end{equation}

After obtaining the posterior $\pi_k$, the agent declares a classification $\hat{z}_k\in \{0,1\}$. Under a symmetric 0-1 loss, the optimal classifier is the maximum a posteriori (MAP) rule:
\begin{equation}
    \hat{z}_k = \boldsymbol{1}_{\{\pi_k \geq \frac{1}{2}\}}
\end{equation}
and the expected, instantaneous misclassification cost is 
\begin{equation}
    c(\pi_k) = \min(\pi_k, 1-\pi_k)
\end{equation}
We consider an infinite-horizon setting with discount factor $\gamma\in(0,1)$. Using the expensive sensor incurs an additive cost $\lambda > 0$ per use. The agent wants to minimize
\begin{equation}
    J(\pi_0: \mu) = \mathbb{E}^\mu \Big [ \sum_{k=0}^\infty \gamma^k \Big( c(\pi_k) + \lambda \boldsymbol{1}_{\{s_k=E\}}\Big) \Big | \pi_0^-=\pi_0 \Big],
    \label{eq:surrogate-objective}
\end{equation}
where $\mu: [0,1]\rightarrow\{C,E\}$ is the sensor selection policy. The problem is a discounted infinite-horizon POMDP, thus the optimal policy is belief stationary, i.e.,  the belief, $\pi_k$ is a sufficient statistic for the entire history of observations and actions (see e.g Thm 7.6.1 in \cite{Krishnamurthy_2016}). 

The instantaneous cost of using sensor $s$ given belief $\pi$ is the expected classification loss after the Bayesian update plus the sensor usage cost, i.e.,
\begin{equation}
    C(\pi, s) = \lambda \boldsymbol{1}_{\{s=E\}} + L_s(\pi),
\end{equation}
where 
\begin{align}
    \nonumber L_s(\pi) = p^+_s \min[U_s(\pi, 1), 1-U_s(\pi, 1)] + \\ 
    p_s^-\min[U_s(\pi, 0), 1-U_s(\pi, 0)]
    \label{eq:expected-posterior-classification-loss}
\end{align}
is the expected posterior classification error, and  $p_s^+ = \tau_s\pi + \alpha_s (1-\pi)$ and $p_s^- = (1-\tau_s)\pi + (1-\alpha_s)(1-\pi)$ are the observation probabilities. We use the instantaneous cost for defining a myopic policy.

\begin{definition}[Myopic policy] Let  $A(\pi) := L_C(\pi)-L_E(\pi)$ be the myopic advantage function, and $\Pi_s = \{\pi : A(\pi) > \lambda \}$ the expensive sensor region. The  \textit{myopic policy} selects the sensor minimizing the instantaneous cost,
\begin{equation}
    \mu_{\text{m}}(\pi) = \begin{cases}
        E \quad & \pi \in \Pi_s \\
        C \quad & \text{otherwise.}
    \end{cases}
\end{equation}
\end{definition}
Note that, by definition, the expected (immediate) benefit of using the expensive sensor exceeds the immediate sensor cost whenever $\pi\in\Pi_s$.

\begin{lemma}
    The expected posterior classification error, $L_s(\pi)$ is piecewise linear and has two breakpoints $\pi_s^{(0)}$ and $\pi_s^{(1)}$. The breakpoints satisfy $\pi_s^{(1)} < 1/2 < \pi_s^{(0)}$ and consequently, the derivative of $L_s(\pi)$ is
    \begin{equation}
    \frac{dL_s}{d\pi} = \begin{cases}
        +1 \quad & \text{on } (0,\pi_s^{(1)}) \\
        1-\alpha_s-\tau_s \quad & \text{on } (\pi_s^{(1)}, \pi_s^{(0)}) \\
        -1 \quad & \text{on } (\pi_s^{(0)}, 1).
    \end{cases} 
    \label{eq:L-slopes}
\end{equation}
\label{L:L-piecewise-linear-and-derivative}
\end{lemma}
\begin{proof}
    We can use the definition of the posterior update map (\ref{eq:general-update}) and the observation probabilities to rewrite the expected posterior classification loss (\ref{eq:expected-posterior-classification-loss}) and obtain 
\begin{align}
    \nonumber L_s(\pi) = &\min[\pi\tau_s, (1-\pi)\alpha_s] + \\
    &\min[\pi(1-\tau_s), (1-\pi)(1-\alpha_s)].
    \label{eq:expected-posterior-classification-loss-2}
\end{align}
Observe that  $L_s(\pi)$ is piecewise linear in $\pi$  and has two breakpoints,
\begin{align}
    \pi\tau_s &= (1-\pi)\alpha_s \Longrightarrow \pi = \frac{\alpha_s}{\tau_s + \alpha_s} =: \pi_s^{(1)} \\
    \nonumber \pi(1-\tau_s) &= (1-\pi)(1-\alpha_s) \Longrightarrow \\
    \pi&=\frac{1-\alpha_s}{2-\alpha_s-\tau_s} =: \pi_s^{(0)}.
\end{align}
Furthermore, by Assumption \ref{A:sensor-ordering} we have that $\alpha_s < \tau_s$, which means $\pi_s^{(1)} < 1/2 < \pi_s^{(0)}$. Taking the derivative on the intervals defined by $\pi_s^{(0)}$ and $\pi_s^{(1)}$ results in (\ref{eq:L-slopes}). 
\end{proof}
In addition, by Assumption \ref{A:sensor-ordering} we have that $\alpha_E < \alpha_C$ and $\tau_E > \tau_C$, and hence the breakpoints are ordered as 
\begin{equation}
    \pi_E^{(1)} < \pi_C^{(1)} < 1/2 < \pi_C^{(0)} < \pi_E^{(0)}.
    \label{eq:kink-ordering}
\end{equation}
To see that $\pi_E^{(1)} < \pi_C^{(1)}$, note that $\pi_s^{(1)}$ is decreasing in $\tau_s/\alpha_s$ and Assumption \ref{A:sensor-ordering} gives $\tau_E/\alpha_E > \tau_C/\alpha_C$. Similarly, $\pi_s^{(0)}$ is decreasing in $(1-\tau_s)/(1-\alpha_s)$ and $(1-\tau_E) /(1- \alpha_E) < (1-\tau_C) / (1-\alpha_C)$. We thus have the following.

\begin{theorem}
    (Threshold structure of Myopic policy). Under Assumption \ref{A:sensor-ordering}, the myopic advantage function $A(\pi)$ is concave on $[0,1]$. Consequently, the expensive sensor region $\Pi_s$ is either empty (if $\max_\pi A(\pi) \leq \lambda$) or a single open interval $(\underline{\pi}_\lambda, \overline{\pi}_\lambda) \subset (0,1)$.
\end{theorem}

\begin{proof}
    Using Lemma \ref{L:L-piecewise-linear-and-derivative}, we can conclude the myopic advantage function $A(\pi)$ is piecewise linear with breakpoints at $\{\pi_E^{(1)}, \pi_C^{(1)}, \pi_C^{(0)}, \pi_E^{(0)} \}$. Using the ordering (\ref{eq:kink-ordering}) and the slopes (\ref{eq:L-slopes}) we can determine the slope of $A$ in all five intervals. In the first interval $\pi \in (0, \pi_E^{(1)})$, both sensors are in their left-most regime, so $A' = 1-1 = 0$. In the second interval $\pi \in(\pi_E^{(1)}, \pi_C^{(1)})$, Sensor $E$ has passed its first breakpoint and is now in the middle regime with slope $1-\alpha_E-\tau_E$. Sensor $C$ is still in its left-most region with slope $+1$. So $A' = 1-(1-\alpha_E-\tau_E) = \alpha_E + \tau_E > 0$. In the third interval $\pi \in (\pi_C^{(1)}, \pi_C^{(0)})$, both sensors are now in the middle regime. This means $A' = 1-\alpha_C-\tau_C-(1-\alpha_E-\tau_E) = (\alpha_E + \tau_E)-(\alpha_C + \tau_C)$ which can be either positive, negative or zero. However, it satisfies $A' \leq \alpha_E+ \tau_E$ (slope at Interval 2) since $\alpha_C + \tau_C > 0$. In the fourth interval $\pi \in (\pi_C^{(0)}, \pi_E^{(0)})$, Sensor $C$ passed its second breakpoint and now has a slope of $-1$ while Sensor $E$ is in its middle regime. This means $A' = -1 -(1-\alpha_E-\tau_E) = \alpha_E + \tau_E - 2 < 0$. This satisfies $A' < (\alpha_E + \tau_E) - (\alpha_C + \tau_C)$ since $(\alpha_C + \tau_C) < 2$. In the final interval $(\pi_E^{(0)},1)$, both sensors are in their right-most regime where the slope if $L_s$ is $-1$, so $A' = -1-(-1) = 0$.
    
    In summary the derivative of $A(\pi)$ is 
    \begin{equation}
        A'(\pi) = \begin{cases}
            0 \quad & \text{on } (0, \pi_E^{(1)}) \\
            \alpha_E + \tau_E \quad & \text{on } (\pi_E^{(1)},\pi_C^{(1)}) \\
            (\alpha_E + \tau_E)-(\alpha_C + \tau_C) \quad & \text{on } (\pi_C^{(1)},\pi_C^{(0)}) \\
            (\alpha_E + \tau_E)-2 \quad & \text{on } (\pi_C^{(0)},\pi_E^{0)}) \\
            0 \quad & \text{on } (\pi_E^{(0)},1), \\
        \end{cases}
    \end{equation}
    where the three interior intervals form a non-increasing sequence. $A$ is strictly increasing on Interval 2 and strictly decreasing on interval 4. Furthermore, from the definition of $L_s$ in (\ref{eq:expected-posterior-classification-loss}) it follows that $L_s(0) = L_s(1) = 0$ for $s \in \{C,E\}$, which means $A(\pi) = 0$ on $\pi\in[0,\pi_E^{(1)})\cup (\pi_E^{(0)},1]$. Consequently, $A(\pi)$ is concave which implies that the super-level set $\{A > \lambda\}$ is convex, i.e., it is an interval. 
\end{proof}
We can now state the following result.
\begin{theorem}
\label{thm:myopic-thres}
    Consider the POMDP in (\ref{eq:surrogate-objective}) with two actions corresponding to  sensors $E$ and $C$, where
    \begin{enumerate}
        \item The instantaneous cost $C(\pi, s)$ is concave in $\pi$ for each $s$. 
        \item Sensor $E$ Blackwell-dominates Sensor $C$: There exists a stochastic matrix $G$ such that $O_C = O_EG$ (see Lemma \ref{L:Blackwell-dominance}).
        \item The transition probabilities of the hidden state are action-independent (the sensor choice does not affect how $z_k$ evolves). 
    \end{enumerate}
    Let $\Pi_s = \{\pi: C(\pi, E)< C(\pi, C)\}$ be the set of beliefs where the expensive sensor has lower instantaneous cost. Then the optimal policy, $\mu^*(\pi)$ which minimizes (\ref{eq:surrogate-objective}) satisfies $\mu^*(\pi) = E$ for all $\pi\in\Pi_s$.
\end{theorem}
\begin{proof}
    Condition 1 is satisfied since $C(\pi, s)$ is a sum of a constant and $L_s(\pi)$ which is defined as the minimum of affine functions (\ref{eq:expected-posterior-classification-loss-2}) which is concave in $\pi$. Condition 2 is satisfied by Lemma \ref{L:Blackwell-dominance} and Condition 3 is satisfied by construction. See Thm 14.7.1 in \cite{Krishnamurthy_2016} for more details. 
\end{proof}
This means that whenever the myopic policy chooses the expensive sensor, so does the optimal policy. The myopic policy may, however, be using the expensive sensor less frequently than the optimal policy.



\begin{algorithm}[t]
\caption{LASE-AD in the Control Loop}
\begin{algorithmic}[1]
\State $S_0 \gets \mathfrak{Z}$ \Comment{The set of trustworthy sensors}
\State $B \gets \text{FIFO(max\_len=}T_{\text{r}}$)
\While{True}
    \State $y_k^r \gets \text{GetRawSensorValues()}$
    \State $y_k \gets \mathcal{P}_{S_k}(y_k^r)$
    \State $\hat{x}_{k:k-1} \gets \text{EKFPredict}(\hat{x}_{k-1}, u_{k-1})$
    \State $a_k \gets \mathcal{D}(y_k, u_{k-1})$
    \State Update belief $\pi_k$ from (\ref{eq:alert-belief-update}) with $a_k$
    \State $\text{use\_probing}, \mathfrak{z} \gets \text{DecideProbing}(\pi_k)$ \label{algline:probing}\Comment{Returns a boolean and a sensor to probe}
    \If{use\_probing}
        \State $\hat{x}^{S_k \setminus \{ \mathfrak{z} \}} \gets \text{ReEstimateWithoutSensors}(B, \{\mathfrak{z}\})$ \label{algline:reest1}
        \State Obtain $u_k$ by solving (\ref{eq:probing-ctr-signal})
    \EndIf

    \If{probing\_used\_last\_time}
        \State Update belief $\pi_k$ with (\ref{eq:probing-update}) and $\hat{x}^{S_k \setminus \{ \mathfrak{z} \}}$ as alternative hypothesis
    \EndIf
    \State $S_{k+1} \gets \text{DecideTrustableSensors}(\pi_k)$ \label{algline:deciding}
    \If{$S_k \neq S_{k+1}$}
        \State $\hat{x}_{k} \gets \text{ReEstimateWithoutSensors}(B, \mathfrak{Z} \setminus S_{k-1})$ \label{algline:reest2}    
    \Else
        \State $\hat{x}_k \gets \text{EKFUpdate}(\hat{x}_{k:k-1}, y_k)$    
    \EndIf
    \State $B\text{.add}((\hat{x}_k, y_k))$
    \If{!use\_probing}
        $u_k \gets \text{NominalController}(\hat{x}_k)$
    \EndIf
    \State Execute $u_k$ and let system (\ref{eq:sys-dynamics}) evolve.
    \State $\text{probing\_used\_last\_time} \gets \text{use\_probing}$
    \State $k\gets k+1$
\EndWhile
\end{algorithmic}
\label{alg:lase-ad}
\end{algorithm}


\section{Numerical Results }
\label{sec:numerical}
\subsection{Evaluation methodology}
We use an extended version of the CartPole environment for the evaluation~\cite{brockman2016openaigym}, where the goal is to balance a pole on a cart. The state $x_k = [p_k, v_k, \theta_k, \omega_k]^\top$ contains cart position/velocity and pole angle/angular velocity~\cite{bartocartpole}, discretized via the RK4 method. The control input, $u_k\in[-10,10]$ is a horizontal force applied to the cart. The system is equipped with $\vert\mathfrak{Z}\vert=3$ sensors: a wheel encoder, an external camera and an IMU. The wheel encoder (E) outputs raw measurements of the position and velocity of the cart, $y_{E,k}^r=(p_{E,k}, v_{E,k})$. The external camera (C) produces raw measurements $y_{C,k}^r = (p_{C,k}, \theta_{C,k})$ and the IMU (I) gives measurements $y_{I,k}^r = (\dot{v}_{I,k}, \omega_{I,k})$. The three sensors have covariance matrices $R_{\mathfrak{z}}$ for $\mathfrak{z}\in \{ E,C,I \}$. The perception pipeline $\mathcal{P}$ produces soft measurements $y_k = \mathcal{P}(y_{E,k}^r, y_{C,k}^r, y_{I,k}^r)$ from the raw measurements. It computes weighted averages between raw measurements and performs sensor fusion when needed. As an example, the soft measurement for the position $p_{\text{soft},k} = wp_{E,k}^r + (1-w)p_{C,k}^r$ is a weighted average of the raw measurements, $w$ is set to the value yielding the smallest variance. The soft measurement for the velocity $v_{\text{soft},k}$ is computed as a weighted average between the raw measurement from the encoder and a fused estimate between the acceleration from the IMU and the previous estimate. The pipeline $\mathcal{P}$ produces soft measurements that satisfy equation (\ref{eq:soft-measurement-state}) with $C = I_{4\times4}$. The graph $\mathcal{G}$ corresponding to $\mathcal{P}$ is shown in Fig.~\ref{fig:bayesiannet}. In the figure, $z_{1,k}$ is the encoder, $z_{2,k}$ the camera and $z_{3,k}$ the IMU; $s_{1,k}$ is the position, $s_{2,k}$ the velocity, $s_{3,k}$ the pole angle and $s_{4,k}$ the angular velocity. 
$\mathcal{D}$ is the CUSUM detector from \cite{savior}, tuned on benign data to achieve a false positive rate of under 5\%. The false negative rate used in (\ref{eq:bayesnet-fp-fn}) is set to 30\%. The cost matrices in (\ref{eq:main-opt-problem}) are $M_x = \text{diag([1 1 20 2])}$ and $M_u = 1$.

We explore two approaches for computing the thresholds for active learning. The first approach is to simulate attacks against sensors with a stochastic attacker. We set the attack transition probabilities to $\mathbb{P}(z_{i,k+1}=1|z_{i,k}=0) = 0.01$ and $\mathbb{P}(z_{i,k+1}=1|z_{i,k}=1) = 0.99$ and used fixed attack magnitudes. The thresholds are tested against the stochastic attacker and optimized with Bayesian optimization (BO). We call this approach \textbf{LASE-AD-S}(tochastic). The other approach is to optimize the thresholds with BO on benign data, denoted by \textbf{LASE-AD-B}(enign). The resulting thresholds from BO was to probe when $\pi\in(0.5, 0.59)$ for LASE-AD-S and when $\pi \in (0.499..., 0.5)$ for LASE-AD-B. This means LASE-AD-B will only use the passive part of LASE-AD and never probe. 



{\bf Baselines} We compare our work against five baselines. All baselines use the same LQR as nominal controller as the LASE-AD variants. \textbf{Normal}: no defense, acts on attacked measurements. \textbf{WoLF-IMQ, WoLF-MD, WoLF-TMD}: outlier-robust Kalman filter variants from \cite{outlier-robust-ekfs}, which scale the Kalman gain by a weighting function. The \textbf{WoLF}-methods have a hyperparameter and this is tuned for each attack scenario with BO. \textbf{KalmanPred}: upon CUSUM detection, this method ignores all sensor readings and relies on EKF prediction. An oracle signals when the attack ends, giving this baseline an advantage. 

{\bf Attack Scenarios} We consider three attack scenarios. The scenario EncoderAttack($t$) is an attack on the wheel encoder during $t$ seconds. The attack is an additive term of 0.5 that is added to  $p_{E,k}$ and $v_{E,k}$ coming from the encoder. The second scenario is called Encoder-IMUAttack, here the same bias is injected to the encoder when $t\in[3.0,6.0]$. A bias is also injected to the IMU, an additive term of 0.9 is added to $\omega_{I,k}$ and 0.2 is added to $\dot{v}_{I,k}$ when $t\in[4.0,7.0]$. This means the attacks overlap for 2.0 seconds. The final and most complex scenario is the E(ncoder)I(MU)C(amera)Attack where all sensors are compromised at some point. The encoder is attacked for $t\in[3.0, 4.0]$, then the IMU for $t\in[4.0,6.0]$ and lastly the camera for $t\in[6.0,7.0]$. The injected biases are the same as in EncoderAttack($t$) and Encoder-IMUAttack for the encoder and IMU while the camera is attacked by adding 0.3 to $p_{C,k}$ and 0.15 to $\theta_{C,k}$. We refer to the scenario without attack as \emph{NoAttack}. The system is simulated for a total of $10$ seconds and time is discretized by a step size of $5$~ms.

\subsection{Experimental Results}
Fig. \ref{fig:failure-rate-bars} shows the task failure rate for the LASE-AD variants and the baseline methods, computed over $50$ evaluations. A task is considered failed if $|\theta_k| > 90^\circ$ at any point during the simulation, i.e, if the pole falls. The failure rate is $0$ for \textit{NoAttack} using all methods, hence we omit the scenario in the figure. We observe that the two LASE-AD variants achieve a significantly lower failure rate than the baselines across all scenarios, and LASE-AD-S outperforms the passive LASE-AD-B variant, especially in the two multi-sensor scenarios. The gap widens with attack duration; under EncoderAttack(0.5) some baselines do not fail, but in the attack scenarios with longer duration, all baselines fail 100\%. The baselines degrade because both the \textbf{WoLF}-methods and \textbf{KalmanPred} essentially rely on EKF predictions when measurements are deemed unreliable, and these predictions drift as the attacks persist. LASE-AD avoids this by identifying and excluding compromised sensors, allowing the filter to still benefit from trustworthy measurements. 

\begin{figure}
    \centering
    \includegraphics[width=1.0\linewidth]{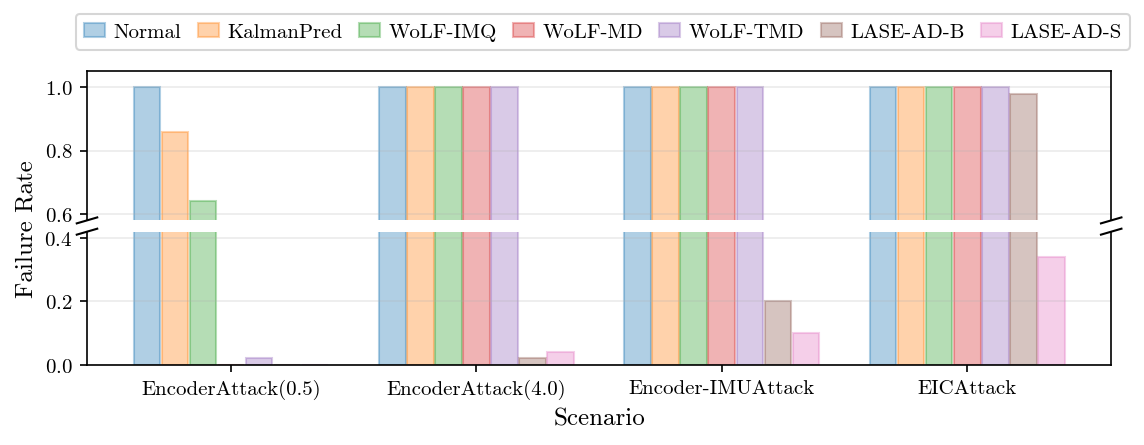}
    \caption{Task failure rate for $4$ attack scenarios, computed over 50 simulations.}
    \label{fig:failure-rate-bars}
\end{figure}
Fig.~\ref{fig:control-error-boxes} shows the control cost distribution over non-failed evaluations. Methods not included had 100\% failure rate for the respective scenario. Under benign conditions, we observe that LASE-AD does not result in increased control error compared to the baselines, i.e., it avoids failures without increased control error when there is no attack. An increase in control error can be seen in the multi-sensor attack scenarios compared to the others. This is partially attributed to the compromised sensors being relatively more difficult to distinguish, but also because the system acts on noisier measurements from perception pipelines using fewer sensors.
\begin{figure}
    \centering
    \includegraphics[width=1.0\linewidth]{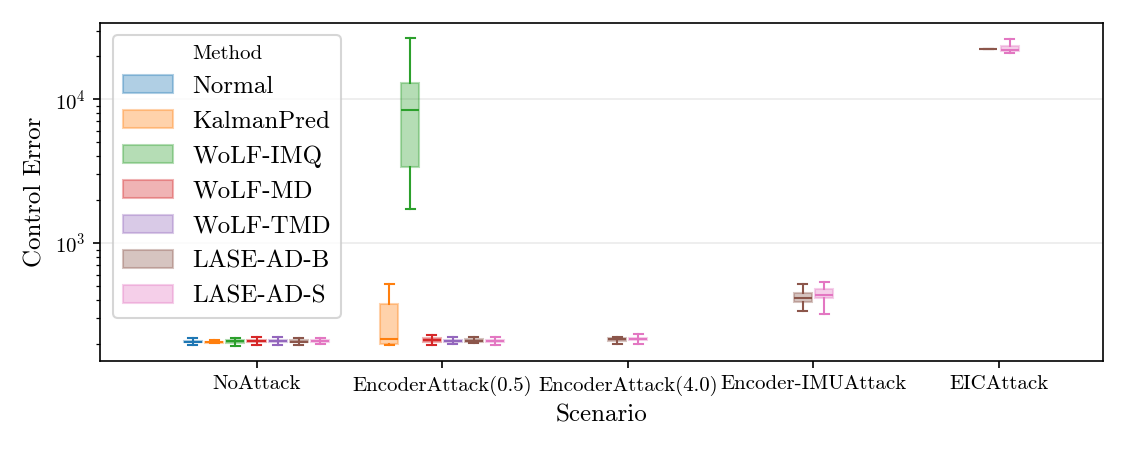}
    \caption{Control error distribution for $5$ scenarios, for non-failed tasks. The boxes show median, inter-quartile range and extreme values.}
    \label{fig:control-error-boxes}
\end{figure}
\begin{figure}
    \centering
    \includegraphics[width=1.0\linewidth]{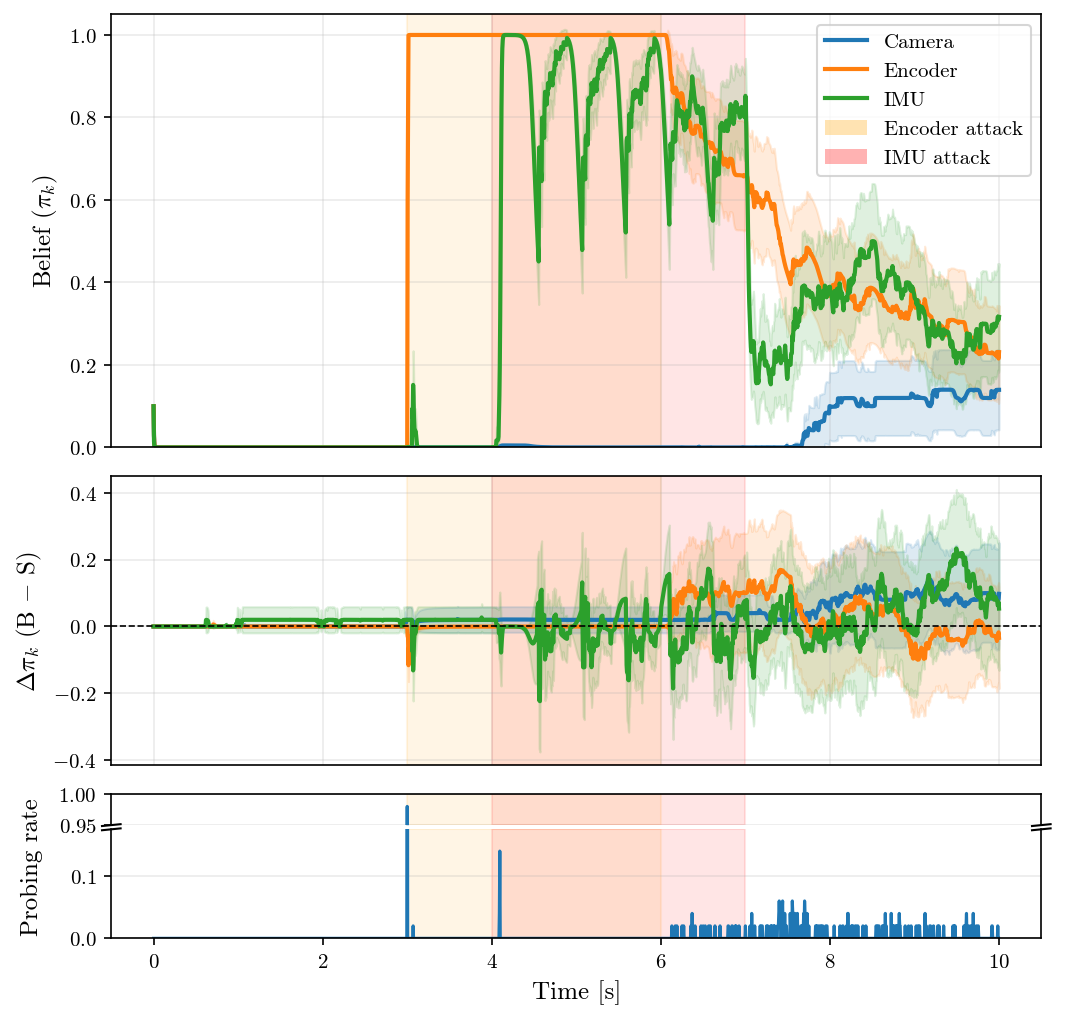}
    \caption{Top: Average belief over time for the LASE-AD-S method. Middle: Average belief difference between LASE-AD-B and LASE-AD-S. Bottom: Rate of usage of the probing control signal for LASE-AD-S.}
    \vspace{-19pt}
    \label{fig:belief-encoder-attack}
\end{figure}

Fig~\ref{fig:belief-encoder-attack} illustrates the belief dynamics in the Encoder-IMUAttack scenario. The top panel shows that LASE-AD-S rapidly identifies both attacked sensors, though the IMU belief exhibits more variability due to ambiguity caused by overlapping attacks on sensors sharing state components (see Fig.~\ref{fig:bayesiannet}). The middle and bottom panels show where active probing improves the belief. Between $t=6.0$ and $t=7.0$, after the encoder attack has ended, LASE-AD-B still assigns high attack probability to the encoder because alert-based inference alone cannot quickly distinguish between the ongoing IMU attack from a continued encoder attack. On the contrary, LASE-AD-S, which probes during the uncertain period (bottom panel), correctly lowers the encoder belief sooner. Recall that LASE-AD-B is a passive variant, so the comparison between the variants directly quantifies the marginal value of active probing.

\section{CONCLUSIONS AND FUTURE WORK}
\label{sec::conclusion}
We have presented LASE-AD, a framework combining passive Bayesian alert interpretation with active learning to bridge the gap between sensor attack detection and recovery. Experiments on a sensor-extended version of the CartPole environment showed that LASE-AD substantially outperforms existing robust state estimation baselines, particularly under prolonged and multi-sensor attacks where methods relying on EKF predictions show degraded performance. 

Important directions of future work include extending the framework to learn detector false negative rates online and explore adaptive probing strategies. Validation on real-world systems, such as unmanned aerial and ground vehicles, is also an interesting direction for future work.

\bibliographystyle{unsrt}
\bibliography{references.bib}
\end{document}